\documentclass[conference]{IEEEtran}
\IEEEoverridecommandlockouts

\usepackage[utf8]{inputenc}
\usepackage[T1]{fontenc}
\usepackage{url}
\usepackage{booktabs}
\usepackage{amsfonts}
\usepackage{nicefrac}
\usepackage{microtype}
\usepackage{amsmath}
\usepackage{graphicx}
\usepackage{algorithm}
\usepackage{algorithmic}
\usepackage{bbm}
\usepackage{mathtools}
\usepackage{caption}
\usepackage{subcaption}
\usepackage{tabularx}
\usepackage{comment}
\usepackage{hyperref}
\usepackage{amsthm}
\usepackage{xcolor}
\newtheorem{theorem}{Theorem}
\newtheorem{definition}{Definition}
\newtheorem{lemma}[theorem]{Lemma}

\def\BibTeX{{\rm B\kern-.05em{\sc i\kern-.025em b}\kern-.08em
    T\kern-.1667em\lower.7ex\hbox{E}\kern-.125emX}}
\begin{document}
\title{Implicit Hypergraph Neural Network}

\author{
\begin{minipage}[t]{0.27\textwidth}
    \centering
    Akash Choudhuri$^{*}$ \\
    \textit{Dept of Computer Science, University of Iowa} \\
    akash-choudhuri@uiowa.edu
\end{minipage}%
\hfill

\vspace{1em}

\begin{minipage}[t]{0.27\textwidth}
    \centering
    Yongjian Zhong$^{*}$ \\
    \textit{Dept of Computer Science, University of Iowa} \\
    yongjian-zhong@uiowa.edu
\end{minipage}%
\hfill
\begin{minipage}[t]{0.27\textwidth}
    \centering
    Bijaya Adhikari \\
    \textit{Dept of Computer Science, University of Iowa} \\
    bijaya-adhikari@uiowa.edu
\end{minipage}
\thanks{$^{*}$ Equal contribution.}
}
\maketitle

\begin{abstract}
Hypergraphs offer a generalized framework for capturing high-order relationships between entities and have been widely applied in various domains, including healthcare, social networks, and bioinformatics. Hypergraph neural networks, which rely on message-passing between nodes over hyperedges to learn latent representations, have emerged as the method of choice for predictive tasks in many of these domains. These approaches typically perform only a small number of message-passing rounds to learn the representations, which they then utilize for predictions. The small number of message-passing rounds comes at a cost, as the representations only capture local information and forego long-range high-order dependencies. However, as we demonstrate, blindly increasing the message-passing rounds to capture long-range dependency also degrades the performance of hyper-graph neural networks.

Recent works have demonstrated that implicit graph neural networks capture long-range dependencies in standard graphs while maintaining performance. Despite their popularity, prior work has not studied long-range dependency issues on hypergraph neural networks. Here, we first demonstrate that existing hypergraph neural networks lose predictive power when aggregating more information to capture long-range dependency. We then propose \textsc{Implicit Hypergraph Neural Network (IHNN)}, a novel framework that jointly learns fixed-point representations for both nodes and hyperedges in an end-to-end manner to alleviate this issue. Leveraging implicit differentiation, we introduce a tractable projected gradient descent approach to train the model efficiently. Extensive experiments on real-world hypergraphs for node classification demonstrate that IHNN outperforms the closest prior works in most settings, establishing a new state-of-the-art in hypergraph learning. 
\end{abstract}


\section{Introduction}
Hypergraphs~\cite{antelmi2023survey} are a generalization of graphs to higher-order relationships. Graphs, where edges can only connect two nodes, capture pairwise relationships. On the other hand, hyperedges in hypergraphs are free to connect an arbitrary number of nodes, enabling them to define higher-order relations involving multiple nodes. The increased expressiveness of hypergraphs has found applications in various domains, including healthcare (where multiple patients are housed in a single inpatient room)~\cite{choudhuri2025domain}, social networks (many users subscribe to a group/subreddit/channel)~\cite{li2013link}, bio-informatics~\cite{tian2009hypergraph}, and cyber-security~\cite{lin2024hypergraph}.

The staggering success of Graph Neural Networks (\textsc{GCNs}) ~\cite{kipf2016semi}, designed for predictive tasks over graphs, has also inspired convolutional neural networks over hypergraphs~\cite{feng2019hypergraph}. Hypergraph neural networks (\textsc{HNNs}) resemble standard \textsc{GCNs} but pass node-embedding messages over hyperedges instead of regular edges. This approach has enabled a number of applications, including disease incidence prediction~\cite{choudhuri2025domain,anand2024h2abm} and echo chamber detection~\cite{hickok2022bounded}, in the domains mentioned above.  However, as generalizations of graph neural networks, \textsc{HNNs} inherit some weaknesses unique to graph neural networks. 
\begin{figure*}[t]
\centering
\begin{subfigure}{0.32\textwidth}
    \centering
    \includegraphics[width=1\textwidth]{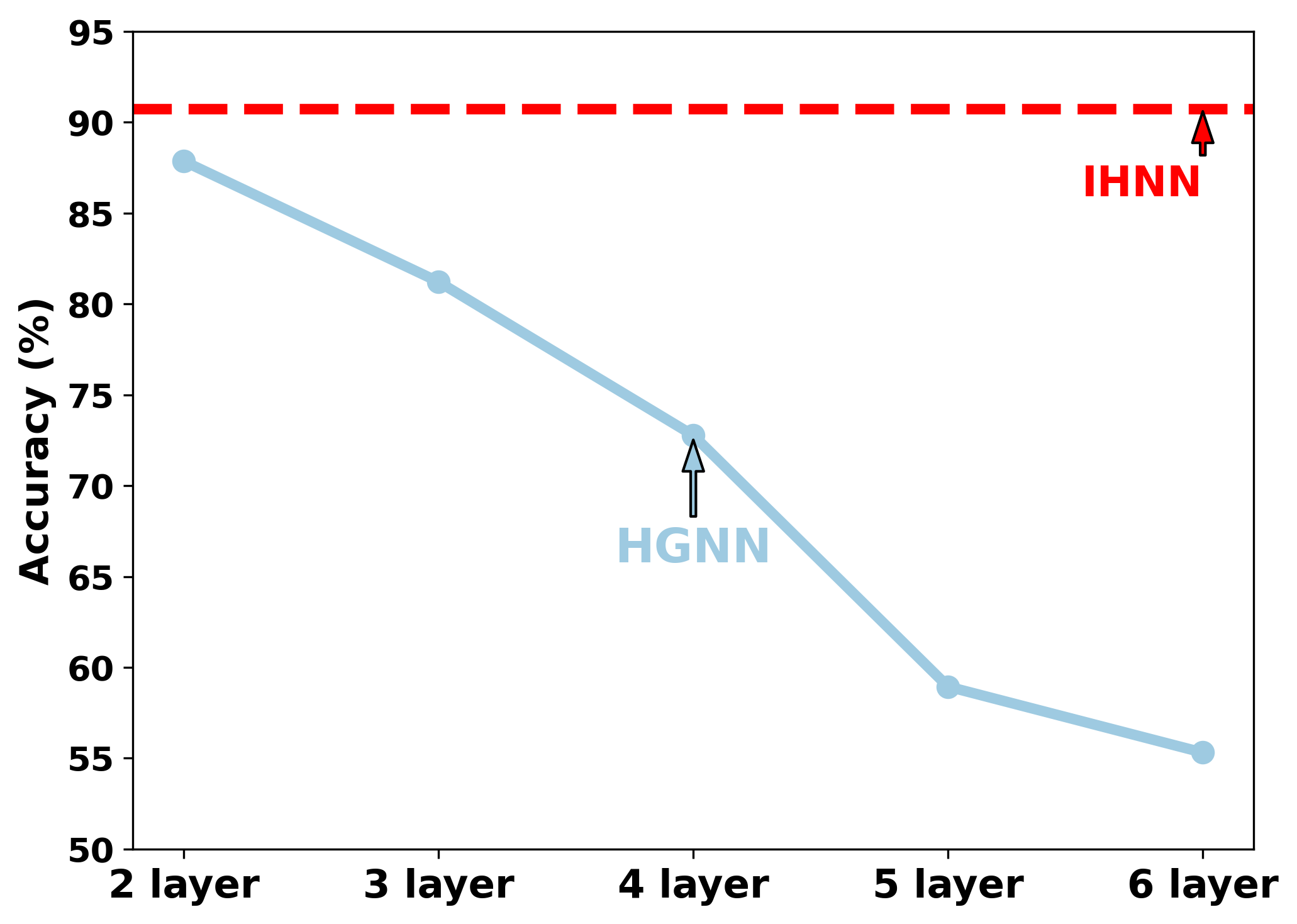}
    \label{fig:subfig1}
\end{subfigure}
\hfill
\begin{subfigure}{0.32\textwidth}
    \centering
    \includegraphics[width=1\textwidth]{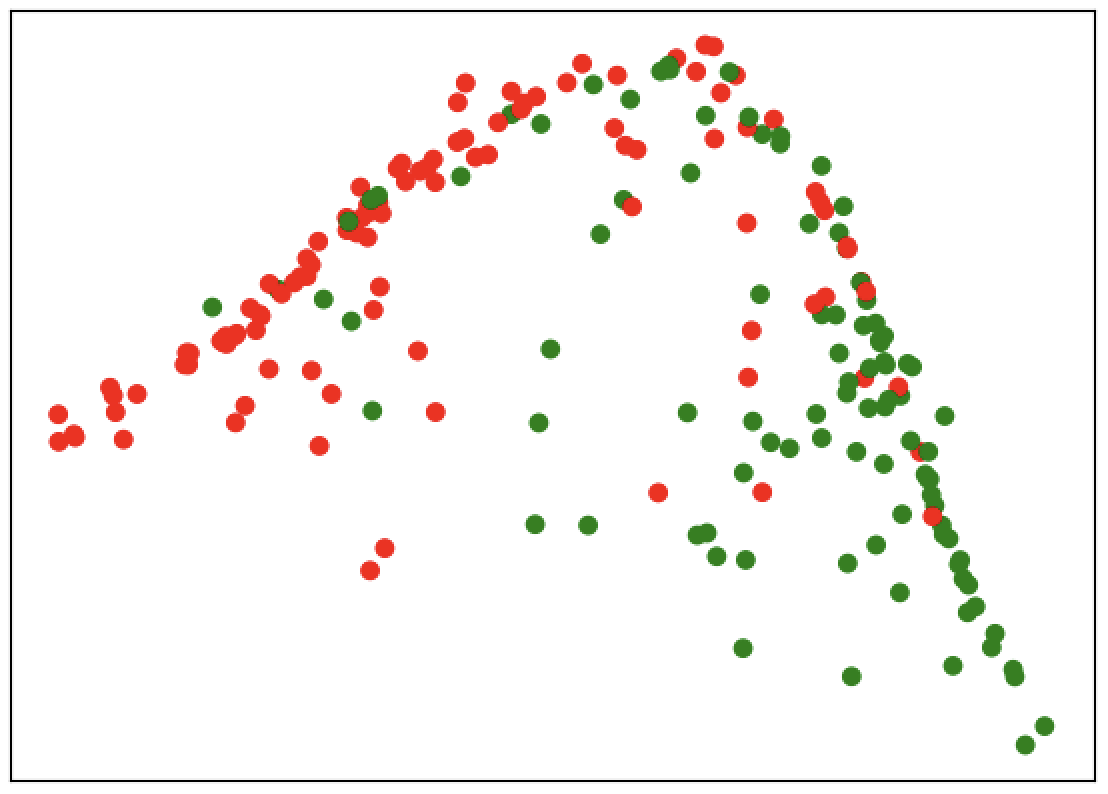}
    \label{fig:subfig2}
\end{subfigure}
\hfill
\begin{subfigure}{0.32\textwidth}
    \centering
    \includegraphics[width=1\textwidth]{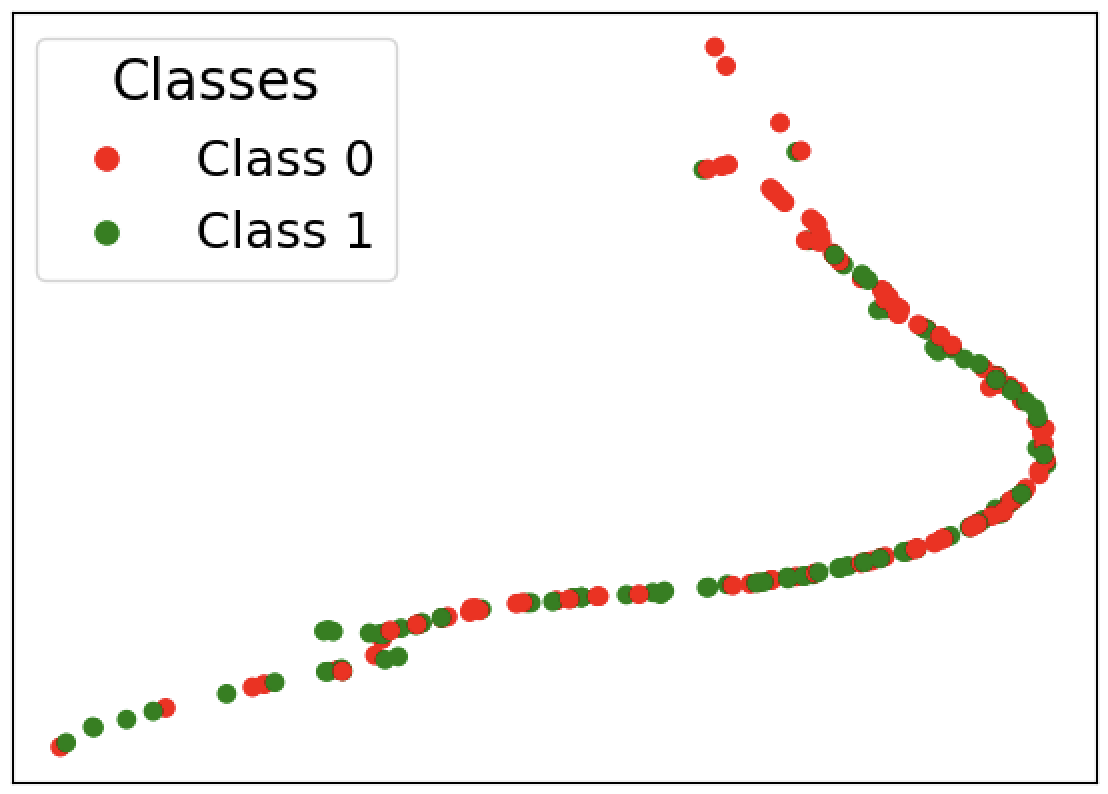}
    \label{fig:subfig3}
\end{subfigure}
\caption{ Best viewed in color.  \textbf{Left:} Performance of \textsc{HGNN} (solid blue line) deteriorates as more layers are stacked on top of each other to capture long-range dependency, while the proposed approach \textsc{IHNN} (red dashed line), which inherently captures long-range dependency, performs the best.  \textbf{Center: }T-SNE visualization of node embeddings for the 2-layer HGNN. \textbf{Right:} The same for 5-layer HGNN. The higher number of layers leads to indiscriminative representations.}
\label{fig:figure1}
\end{figure*}

One such limitation is the fixed radius of information aggregation. \textsc{GCNs} learn representations of nodes by convolving over neighborhoods of a fixed diameter $k$.  Typically, $k$ is smaller than the diameter of the network, hence, GCNs cannot capture long-range dependencies between the nodes that go beyond the radius $k$.  Stacking multiple layers of GCNs to capture relationships between nodes far away in the network can increase the information aggregation radius; unfortunately, it also leads to performance degradation~\cite{chen2020measuring}. This is because as more neighborhood information is aggregated, the node embeddings appear to smooth out, resembling each other more and more and losing discriminative power.

\textsc{HNNs} also convolve over neighborhoods as described by the hyperedges to learn node representations. Hence, the question of whether \textsc{HNNs} can maintain performance while increasing the convolution radius to capture long-range dependency is valid. Here, we first put this question to the test. Through our experiments, we discovered that stacking multiple layers of \textsc{HNNs} leads to performance degradation similar to GCNs (See Figure~\ref{fig:figure1}). Intuitively, this pattern is expected as message passing rounds in \textsc{HNNs} typically involve more nodes than standard \textsc{GCNs}. To demonstrate this, we started by running Hypergraph Neural Network~\cite{feng2019hypergraph} for a fixed radius of 2 on the \textbf{senate-bills} dataset~\cite{chodrow2021generative,fowler2006legislative} and kept track of the performance. We repeated the same experiment while increasing the convolution radius and noticed a degradation pattern  (solid blue line) similar to that of a standard \textsc{GCNs}.  

Implicit Graph Neural Networks~\cite{gu2020implicit}, \textsc{IGNNs}, overcome this issue by learning a \textsc{fixed point} representation of nodes by iterating \textsc{GCN}-like models until convergence. Since the number of iterations can be arbitrarily large, these models can capture information from the entire graph. Empirical evidence shows that these approaches overcome the long-range dependency issues while maintaining performance. 

Despite their importance in numerous applications, no prior work has studied mitigating long-range dependency issues in \textsc{HNNs}. Here, we propose \textsc{Implicit Hypergraph Neural Network} (\textsc{IHNN}), the first implicit model designed for hypergraphs. \textsc{IHNN} explicitly models the node-hyperedge relationship, a key aspect of hypergraph learning, by jointly learning embeddings for both nodes and hyperedges. Additionally, \textsc{IHNN} incorporates a novel membership regularization technique that enhances representation learning and leverages implicit differentiation for training, and it also provably converges to fixed-point embeddings. The key contributions of our paper are as follows:
\begin{itemize}
    \item We propose a novel implicit hypergraph model, \textsc{Implicit Hypergraph Neural Network} (\textsc{IHNN}), which captures long-range dependency between nodes in hypergraphs while maintaining the predictive performance.
    \item To train \textsc{IHNN}, we propose a novel membership regularization technique and a tractable projected gradient descent approach that leverages implicit differentiation. This approach ensures the convergence of \textsc{IHNN}.
    \item We conduct extensive experiments on five real-world datasets of various sizes to demonstrate the state-of-the-art performance on \textsc{IHNN}. We also conduct a case study using real EHR data to demonstrate the broader applicability of IHNN.
    
\end{itemize}

\section{Related Works}
\par\noindent\textbf{Machine Learning on Hypergraphs:}
ML on hypergraph has garnered significant attention for its ability to model complex higher-order relationships within data~\cite{antelmi2023survey}. In this paper, we focus on hypergraph node classification, a fundamental task in hypergraph learning.   Feng et al.~\cite{feng2019hypergraph} introduced the first message-passing neural network tailored for hypergraph learning, laying the foundation for subsequent advancements. Yadati et al.~\cite{yadati2019hypergcn} proposed a novel transformation technique that maps hypergraphs to conventional graphs, facilitating the application of standard graph neural networks. To better capture the importance of high-cardinality hyperedges and high-degree nodes, Dong et al.~\cite{dong2020hnhn} develop a normalization scheme that enhances representation learning. Bai et al.~\cite{bai2021hypergraph} further refine message passing in hypergraph neural networks by incorporating attention mechanisms. Expanding on the flexibility of hypergraph learning, Chien et al.~\cite{chienyou} introduce a highly general framework that integrates the Set Transformer~\cite{lee2019set} with hypergraph neural networks, enabling more expressive representation learning. This is then followed by recent works be Wang et al.~\cite{wang2022equivariant} that applies diffusion operators on hypergraphs and Wang et al.~\cite{wang2023hypergraph} that defines parametrized hypergraph-regularized energy functions and demonstrate how their minimizes serve as learnable node embeddings. 

This growing body of research underscores the potential of hypergraph-based models in capturing intricate structural dependencies beyond what traditional graph models can achieve. This has led to the deployment of hypergraphs in several applications~\cite{xu2023hypergraph,choudhuri2025domain,la2022music,wang2021session}.

\par\noindent\textbf{Graph Implicit Models:}
Implicit models define their outputs using fixed-point equations. Bai et al.~\cite{bai2019deep} proposed an equilibrium model for sequence data, where the solution is obtained from a fixed-point equation. The Implicit Graph Neural Network (IGNN)~\cite{gu2020implicit} extends this idea to graph data and demonstrates implicit model can capture long-range dependencies on graphs. Park et al.~\cite{park2021convergent} introduce an equilibrium-based GNN with a linear transition map, ensuring contraction to guarantee the existence and uniqueness of a fixed point. Liu et al.~\cite{liu2021eignn} developed an infinite-depth GNN that captures long-range dependencies while bypassing iterative solvers by deriving a closed-form solution. El Ghaoui et al.~\cite{el2021implicit} proposed a general implicit deep learning framework, addressing the well-posedness of implicit models. Meanwhile, Chen et al.~\cite{chen2022optimization} formulated their model using the diffusion equation as an equilibrium constraint and solved a convex optimization problem to determine the fixed point.
Zhong et al.~\cite{zhong2024efficient} proposed an implicit model for dynamic graph learning and a bilevel training algorithm, which can efficiently capture the long-range dependency on dynamic graphs and subgraphs~\cite{zhongimplicit}.

\section{Method}
\begin{figure*}[ht!]
    \centering
    \includegraphics[width=.9\textwidth]{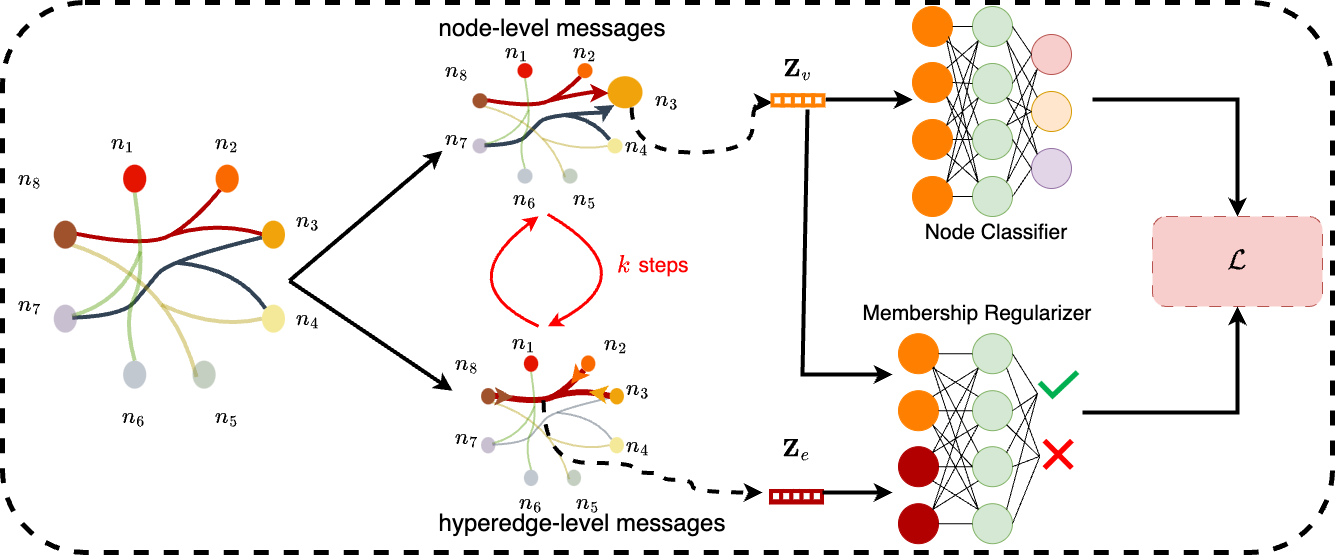}
    \caption{Proposed IHNN Architecture: The framework cyclically obtains the fixed-point node and hyperedge embeddings by iterating the proposed model. The membership regularizer samples a set of node-hyperedge pairs from the training data and performs binary classification to predict if the node is present in the hyperedge or not. Additionally, the node classification loss is computed, and the overall loss $\mathcal{L}$ is the weighted sum of the node classification and membership regularization losses.}
    \label{fig:model}
\end{figure*}

\subsection{Preliminaries}
\par\noindent\textbf{Hypergraph:} A hypergraph is a generalization of a graph where an edge can connect any number of vertices. Formally, a hypergraph $\mathcal{H}$ is defined as a pair: $\mathcal{H}=(\mathcal{V}, \mathcal{E})$, where $\mathcal{V}$ is a set of nodes, i.e., $\mathcal{V} = \{v_1, v_2, \dots, v_n\}$ and $\mathcal{E}$ is a set of hyperedges, where each hyperedge $e \in \mathcal{E}$ is a subset of $\mathcal{V}$. That is, $\mathcal{E} \subset \mathcal{P}(\mathcal{V})$, where $\mathcal{P}(\mathcal{V})$ is the power set of $\mathcal{V}$. Hypergraphs are commonly represented by an incidence matrix \( \mathbf{H} \in \mathbb{R}^{|\mathcal{V}| \times |\mathcal{E}|} \), where each entry \( \mathbf{H}_{ab} \) is 1 if node \( a \) is incident upon hyperedge \( b \), and 0 otherwise.
The node degree matrix \( \mathbf{D}_{v} \in \mathbb{R}^{|\mathcal{V}| \times |\mathcal{V}|} \) and the hyperedge degree matrix \( \mathbf{D}_{e} \in \mathbb{R}^{|\mathcal{E}| \times |\mathcal{E}|} \) are diagonal matrices where the diagonal entry \( [\mathbf{D}_{v}]_{ii} \) represents the number of hyperedges incident on node \( v_i \), while \( [\mathbf{D}_{e}]_{ii} \) represents the cardinality of hyperedge \( e_i \).

In the current setup, we are given a node feature matrix $\mathbf{X} \in \mathbb{R}^{|\mathcal{V}|\times d}$, where $d$ is the number of dimensions. Each row of $\mathbf{X}$ corresponds to the feature vector for each node $v_i \in \mathcal{V}$. When no feature matrix is given, the learning procedure can be extended to include feature initialization separately; or the features can be randomly initialized. Additionally, we are also provided with the node label set $\mathcal{Y}$ where each node $v_i \in \mathcal{V}$ has a label $y_i \in \mathcal{Y}$.

\par\noindent\textbf{Problem Setting:} In this work, we focus on a transductive node classification problem where the nodes in the training set $\mathcal{V}_{\text{train}}$ and their corresponding labels $\mathcal{Y}_{\text{train}}$ is revealed to the model during the training. The model is then tasked with predicting the labels of the nodes in the test set $\mathcal{V}_{\text{test}}$. Note that $\mathcal{V}_{\text{train}} \cap \mathcal{V}_{\text{test}}=\emptyset$. Note that although the labels of the nodes in the test set are hidden, the overall structure of the hypergraph remains visible to the model during training in the transductive setting. Having defined all the notations, our overall problem statement is defined as follows:

\medskip
\noindent\fbox{%
    \parbox{0.97\columnwidth}{%
\textbf{Given:} A hypergraph $\mathcal{H} = (\mathcal{V},\mathcal{E})$, with feature matrix $\mathbf{X}$  and a set of labeled nodes $\mathcal{V}_{\text{train}}$, corresponding labels $\mathcal{Y}_{\text{train}}$, and a set of unlabeled nodes $\mathcal{V}_{\text{test}}$.\\
\textbf{Infer:} A predictor $f(\cdot)$ which predicts the node labels.
}
}

\subsection{Implicit Hypergraph Neural Network}
\par\noindent\textbf{Graph Implicit Models:} Implicit models~\cite{el2021implicit,geng2021training,gu2020implicit} typically learn representations by finding the fixed point solution to a function \( f(\mathbf{Z},\mathbf{X}) \), where \( f(\cdot) \) is typically a universal function approximator, such as a neural network. Here \( \mathbf{X} \) represents the input features, and \( \mathbf{Z} \) denotes the learned embeddings. The fixed-point embeddings are obtained iteratively as follows:

\begin{equation}
   \mathbf{Z}^* = \lim_{k\rightarrow\infty} \mathbf{Z}_{k+1} = \lim_{k\rightarrow\infty} f(\mathbf{Z}_k, \mathbf{X}) = f(\mathbf{Z}^*, \mathbf{X}) 
\end{equation}

Such a function $f(\cdot)$ is said to be well-posed if there is a fixed-point solution to it (i.e., it produces converged representations that do not change on further iteration). To ensure well-posedness, \( f(\cdot) \) is typically required to be a contractive mapping~\cite{gu2020implicit}.

\begin{definition}[Contractive Function]
    A function $f(\cdot)$ is called $\mu$-contractive, for $0<\mu<1$, if it satisfies $\|f(x)-f(y)\|\le \mu\|x-y\|$ for any $x,y \in \text{dom}~f$.
\end{definition} 
\begin{definition}[Well-posedness]
    Let $f_\theta$ be a function parameterized by $\theta$. We claim $\theta$ is well-posed for $f_\theta$ if the equation \( \mathbf{Z}=f_\theta(\mathbf{Z},\mathbf{X}) \), for any \(\mathbf{X}\), admits a unique solution.
\end{definition}

Note that all implicit graph neural networks satisfy the well-posedness property~\cite{gu2020implicit}.


 \par\noindent\textbf{Hypergraph Neural Network:} 
Any implicit hypergraph neural network we design must utilize a graph neural network specifically designed for hypergraphs. Hence, a straightforward approach for our aspirations is to start with an existing hypergraph neural network and define an implicit function over it. Let us examine one of the state-of-the-art hypergraph neural networks (HGNN)~\cite{feng2019hypergraph}, as illustrated in the equation below.
 


\begin{align}
\label{eq:HNN}
    \mathbf{Z}_{t+1} &= \sigma(\mathbf{\mathbf{D}_{v}^{-\frac{1}{2}} \mathbf{H} \mathbf{D}_e^{-1} \mathbf{H}^\top \mathbf{D}_{v}^{-\frac{1}{2}}Z}_t\mathbf{W}).
\end{align}

Here, $\mathbf{Z}_{t}$ is the embedding at the $t^{th}$ layer and \( \mathbf{W} \) is a learnable parameter. The rest of the notations are as defined earlier. Now, for further discussion let \( \mathbf{L} := \mathbf{D}_{v}^{-\frac{1}{2}} \mathbf{H} \mathbf{D}_e^{-1} \mathbf{H}^\top \mathbf{D}_{v}^{-\frac{1}{2}} \). The equation above can now be simplified as $\mathbf{Z}_{t+1} = \sigma(\mathbf{L} \mathbf{Z}_{t} \mathbf{W})$. This formulation treats \( \mathbf{L} \), which is the approximated hypergraph Laplacian matrix~\cite{feng2019hypergraph}, analogously to the adjacency matrix in conventional graphs. Since the function \( \sigma \) is a non-expansive activation function (e.g., ReLU, sigmoid), the formulation leads to a meaningful implicit graph neural network over hypergraphs. However, a key drawback of this formulation is that it solely focuses on learning node embeddings. It disregards the key distinction between hypergraphs and conventional graphs, the varying cardinality of hyperedges. Consequently, developing an implicit model based on Equation \ref{eq:HNN} can only yield suboptimal results (as validated by our experiments).

\par\noindent\textbf{Our Method:} The limitation of the implicit model based on equation \ref{eq:HNN} indicates that a richer interplay between nodes and hyperedges can be further exploited. To achieve this, the method we design must also maintain hyperedge embeddings, in addition to learning node embeddings. To this end,  we start by rewriting the hypergraph Laplacian matrix as follows:

$$\mathbf{L} = \mathbf{L}_{ve} \mathbf{L}_{ve}^T \text{, where } \mathbf{L}_{ve} := \mathbf{D}_{v}^{-\frac{1}{2}} \mathbf{H} \mathbf{D}_e^{-\frac{1}{2}} $$

In this form, \( \mathbf{L}_{ve} \) serves as a transformation from node space to hyperedge space, while \( \mathbf{L}_{ve}^T \) does the reverse. Leveraging this insight, we can establish a new formulation that models the interplay between nodes and hyperedges as follows:
\begin{align}
    \label{eq:interplay}
     [\mathbf{Z}_v]_{t+1} = \sigma( \mathbf{L}_{ve}^T[\mathbf{Z}_e]_t \mathbf{W} + b_\Omega(\mathbf{X}) )\\ \nonumber
     [\mathbf{Z}_e]_{t+1} = \sigma\left( \mathbf{L}_{ve}[\mathbf{Z}_v]_t \mathbf{W} + b_\Omega(\mathbf{X}) \right)
\end{align}
Here, \( \mathbf{Z}_v \) and \( \mathbf{Z}_e \) denote the embeddings of nodes and hyperedges, respectively, and \( b_{\Omega}(\cdot) \) is an affine transformation parameterized by \( \Omega \). In this formulation, we update the node embeddings by aggregating information from hyperedge embeddings, and update the hyperedge embeddings by aggregating information from node embeddings iteratively. 

Iterating over $t$ in Equation (\ref{eq:interplay}) leads to fixed-point embeddings for nodes and hyperedges if our formulation is well-posed. Also note obtaining fixed-point embeddings by iterating over Equation (\ref{eq:interplay}) is equivalent to solving the following equation.
\begin{align}
    \label{eq:IHNN_formulation}
    \begin{pmatrix}
        \mathbf{Z}_v\\
        \mathbf{Z}_e
    \end{pmatrix} = & \sigma\left(
    \begin{bmatrix}
        0 & \mathbf{L}_{ve}^T\\
        \mathbf{L}_{ve} & 0
    \end{bmatrix}
    \begin{pmatrix}
        \mathbf{Z}_v\\
        \mathbf{Z}_e
    \end{pmatrix} \mathbf{W} +  b_{\Omega}(\begin{pmatrix}
        \mathbf{X}_v\\
        \mathbf{X}_e
    \end{pmatrix}) \right)
\end{align}
where $\mathbf{X}_v$ is the node features, and $\mathbf{X}_e$ is the hyperedge features. $\mathbf{X}_e$ is initialized as the average of the features of their member nodes in case they are unavailable.

For the well-posedness of our model, we provide the following lemma, which has similar requirements as other implicit models \cite{gu2020implicit, zhong2024efficient}
\begin{lemma}
    Let \(\Bar{\mathbf{A}}=\left[ 0~\mathbf{L}_{ve}^T;~\mathbf{L}_{ve}~0 \right] \), and let \(\|\cdot\|_{\text{op}}\) and \(\|\cdot\|_\infty\) denote operator and infinity norms, respectively. \(\mathbf{W}\) is well-posed for Equation (\ref{eq:IHNN_formulation}) if \(\|\Bar{\mathbf{A}}\|_{\text{op}}  \|\mathbf{W}\|_{\infty}< 1 \) for fixed \(\Bar{\mathbf{A}}\).
\end{lemma}
\begin{proof}
We first rewrite the equation (\ref{eq:IHNN_formulation}) as $Z=\sigma(\Bar{\mathbf{A}}\mathbf{ZW}+\mathbf{B})$. Showing $W$ is well-posed for equation (\ref{eq:IHNN_formulation}) is equivalent to show the following vectorized equation admits a unique solution: $\text{vec}({\mathbf{Z}})=\sigma\left(({\mathbf{W}}^T\otimes \Bar{\mathbf{A}})\text{vec}({\mathbf{Z}})+\text{vec}({\mathbf{B}})\right)$. By Lemma B.1 from \cite{gu2020implicit}, we know that if $\|\mathbf{W}^T\otimes \Bar{\mathbf{A}}\|_{op}<1$ then the above equation has a unique solution, which can be attained by fixed-point iteration. Moreover, this constraint can be further converted to $\|{\mathbf{W}}\|_{\infty}\|\Bar{\mathbf{A}}\|_{op}\le\|{\mathbf{W}}\|_{op}\|\Bar{\mathbf{A}}\|_{op}=\|{\mathbf{W}}^T\otimes \Bar{\mathbf{A}}\|_{op}<1$.
\end{proof}

\par\noindent\textbf{Membership Regularization:} To further capture the relationship between nodes and hyperedges, we introduce a novel membership regularization technique. Inspired by the labeling trick~\cite{zhang2021labeling}, which enhances representation learning by incorporating node membership information, our approach encourages the model to better leverage node-hyperedge interactions. Given fixed-point node embeddings \( \mathbf{Z}_v \) and hyperedge embeddings \( \mathbf{Z}_e \), we impose membership regularization as follows:

\begin{enumerate}
    \item \textbf{Hyperedge Sampling:} Sample a batch of hyperedges.
    \item \textbf{Node Sampling:} For each hyperedge in the batch, sample a member or non-member node.
    \item \textbf{Classification:} Concatenate each hyperedge embedding with the embeddings of its corresponding sampled nodes. These concatenated embeddings are then fed into a classifier to predict whether each sampled node belongs to the given hyperedge. The classification loss serves as membership regularization.
\end{enumerate}

\par \noindent \textbf{Overall Model:} Our overall model is presented in Figure \ref{fig:model}. Let \( \hat{\mathbf{Z}} := (\mathbf{Z}_v, \mathbf{Z}_e) \) and \( \hat{\mathbf{X}} := (\mathbf{X}_v, \mathbf{X}_e) \). Then, combining all components of \textsc{IHNN}, the overall objective function can be formulated as follows:

\begin{align}
    \label{eq:IHNN}
    \min_{\theta,\mathbf{W},\Omega,\phi}\;& 
    \mathcal{L}(\theta,\mathbf{W},\Omega,\phi) := \frac{1}{|\mathcal{V}_{\text{train}}|}
    \sum_{i\in \mathcal{V}_{\text{train}}} 
    \ell_1 \!\left(f_\theta(\hat{\mathbf{Z}}),y_i\right) \\
    &\quad + \gamma \ell_2 \!\left(g_{\phi}(\hat{\mathbf{Z}})\right) \nonumber \\
    \text{s.t.}\;& 
    \hat{\mathbf{Z}} = \sigma\!\left(
    \Bar{\mathbf{A}}\hat{\mathbf{Z}} \mathbf{W} + b_{\Omega}(\hat{\mathbf{X}}) \right), ~ 
    \|\mathbf{W}\|_{\infty} \leq \frac{\kappa}{\|\Bar{\mathbf{A}}\|_{\text{op}}}.
\end{align}

Here, \( \ell_1 \) is the classification loss (e.g., cross-entropy), and \( f_\theta \) consists of two components: (1) a cross-attention mechanism that maps \( (\mathbf{Z}_v, \mathbf{Z}_e) \) to node embeddings, and (2) a classifier that predicts node labels. The function \( g_\phi \) is the membership regularize and \( \ell_2 \) is the regularization loss, \( \gamma \) is a trade-off parameter, and \( \kappa \) is a scalar less than 1 to ensure numerical stability.

\begin{algorithm}[t]
\caption{Training for IHNN}
\label{algo}
\begin{algorithmic}[1]

\REQUIRE Hypergraph \(\mathcal{H}\), feature matrix \(\mathbf{X}\)  
\ENSURE \(\theta, \phi, \mathbf{W}, \Omega\)  

\FOR{\(t = 0,1,\dots,T\)}
    \STATE Compute \(\hat{\mathbf{Z}}\) using fixed-point iteration from Equation (\ref{eq:IHNN_formulation})
    \STATE Sample a batch of hyperedges \(\mathcal{B}_e\) and initialize \(\mathcal{B}_v\) as empty  
    \FOR{\(e \in \mathcal{B}_e\)}
        \STATE With probability \( \frac{1}{2} \), randomly sample a member of \(e\); otherwise, sample a non-member of \(e\)  
        \STATE Add the selected node \(v\) to \(\mathcal{B}_v\)  
    \ENDFOR
    \STATE Feed \([\mathbf{Z}_e(\mathcal{B}_e), \mathbf{Z}_v(\mathcal{B}_v)]\) into classifier \(g_\phi\) and get the regularizer   
    \STATE Feed \([\mathbf{Z}_e, \mathbf{Z}_v]\) into classifier \(f_\theta\) and get the classification loss  
    \STATE Compute the gradients using Equations (\ref{eq:gradient}) and (\ref{eq:gradient_other})  
    \STATE Update parameters and project \(\mathbf{W}\) into the feasible space.  
\ENDFOR  

\end{algorithmic}
\end{algorithm}

\subsection{Training}

To optimize for the objective,  we adopt a projected gradient descent through implicit differentiation, which is a standard training algorithm for implicit models~\cite{gu2020implicit}, enabling direct gradient computation through the fixed point. After updating the parameters, $\mathbf{W}$ will be projected back to the feasible set, for which efficient methods exist \cite{duchi2008efficient}. This projection is also used in other implicit models, such as \cite{gu2020implicit}. 

The gradients with respect to $\theta$ and $\phi$, denoted as $\frac{\partial \mathcal{L}}{\partial \theta}$ and $\frac{\partial \mathcal{L}}{\partial \phi}$ respectively, are readily computed as they do not require backpropagation through the fixed point $\hat{\mathbf{Z}}$.  For other variables, we compute gradients by solving the following fixed-point equation:

\begin{align}
\label{eq:gradient}
\frac{\partial \mathcal{L}}{\partial \hat{\mathbf{Z}}} = \mathbf{D} \odot \left( \hat{\mathbf{A}} \frac{\partial \mathcal{L}}{\partial \hat{\mathbf{Z}}} \mathbf{W} + \bar{\nabla}_{\hat{\mathbf{Z}}} \mathcal{L} \right)
\end{align}
where $\bar{\nabla}_{\hat{\mathbf{Z}}} \mathcal{L}$ represents the gradient backpropagated to $\hat{\mathbf{Z}}$ via automatic differentiation.  $\mathbf{D} = \text{diag}\left(\sigma'\left(\bar{\mathbf{A}} \hat{\mathbf{Z}} \mathbf{W} + b_{\Omega}(\hat{\mathbf{X}}) \right)\right)$, with $\sigma'$ denoting the derivative of the activation function $\sigma$, and $\odot$ representing element-wise multiplication. Once the solution for $\frac{\partial \mathcal{L}}{\partial \hat{\mathbf{Z}}}$ is obtained, the remaining gradients can be computed as follows:

\begin{align}
\label{eq:gradient_other}
\frac{\partial \mathcal{L}}{\partial \mathbf{W}} = \hat{\mathbf{Z}}^T \hat{\mathbf{A}} \frac{\partial \mathcal{L}}{\partial \hat{\mathbf{Z}}} ,~
\frac{\partial \mathcal{L}}{\partial \Omega} = \left\langle \frac{\partial b_{\Omega}(\hat{\mathbf{X}})}{\partial \Omega}, \frac{\partial \mathcal{L}}{\partial \hat{\mathbf{Z}}} \right\rangle
\end{align}

Our overall training approach is presented in Algorithm \ref{algo}. The total number of iterations is \( T \). We begin by computing the fixed-point embeddings \( \hat{\mathbf{Z}} \), followed by sampling a batch of hyperedges \( \mathcal{B}_e \) and corresponding nodes \( \mathcal{B}_v \). The model then processes the sampled batch to compute the membership regularization loss. Additionally, classification is performed using fixed-point embeddings. We then compute gradients and update the parameters.

\subsection{Complexity Analysis}
The complexity of our method primarily arises from computing the embeddings, with a base complexity of $O(Cd^2 + (n + E)d)$, where $C$ represents the sum of the number of nodes incident to each hyperedge, $d$ is the hidden dimension, $E$ is the number of hyperedges, and $n$ is the number of nodes. To compute the embeddings, our model requires two fixed-point iteration processes—one during the forward pass and one during the backward pass—introducing an additional factor of $T$, the number of iterations. This results in a complexity of $O(Cd^2T + (n + E)dT)$. For the projection step, we project each row of the weight matrix onto an $\ell-1$ ball, which can be performed efficiently. Therefore, the overall per-iteration complexity of our method is dominated by $O(Cd^2T + (n + E)dT)$.

\section{Experiments}
In this section, we contrast the performance of our proposed approach, \textsc{IHNN}, with state-of-the-art baselines and answer the following research questions: \textbf{RQ1:} How is the performance of \textsc{IHNN} compared to other state-of-the-art hypergraph representation-learning methods? \textbf{RQ2:} What is the individual contribution of each component of \textsc{IHNN}? \textbf{RQ3:} How stable is the performance of \textsc{IHNN} against changes in hyperparameter values? \textbf{RQ4:} Can IHNN find meaningful applications in EHR data? We will first introduce the datasets used in this work to answer these questions and then describe the baselines.
\subsection{Datasets}
\begin{table}[ht!]
    \centering
	\caption{  \label{rel_view_data}
		Summary Statistics of the Datasets used}
	
	\footnotesize
	\begin{tabular}{|c|c|c|c|c|c|}
		\hline
		& \textbf{DBLP} & \textbf{Cora} &\textbf{Walmart-trips} & \textbf{High-school} & \textbf{House-bills} \\
		\hline
		$|\mathcal{V}|$ & $43413$ & $2708$   & $88860$ & $327$ & $1494$\\
		$|\mathcal{E}|$ & $22535$ & $1072$ & $69906$ & $7818$ & $60987$\\
		$|X|$ & $1425$ & $1433$ & $64$ & $64$ & $64$\\
	   $|\mathcal{Y}|$ & $6$ & $7$ & $11$ & $9$ & $2$\\
		$|e|_{max}$ & $202$ & $43$ & $25$ & $148$ & $6220$\\
		\hline
	\end{tabular}
\end{table}
\begin{table*}[ht!]
\centering
\caption{Results for transductive node classification experiments. We report the mean accuracy with standard deviation taken across 10 independent runs. As observed, our proposed approach \textsc{IHNN} performs the best in all datasets.}
\label{tab:main_result}
\centering
\begin{tabular}{|c|c|c|c|c|c|c|c|}
\hline
\textbf{Data} & \textbf{Method} & \textbf{DBLP} & \textbf{Cora} & \textbf{Walmart-trips} & \textbf{High-school} & \textbf{House-bills} & \textbf{Overall Rank}\\ 
Type  &   & co-authorship  & co-authorship  & co-purchases  & contact  & bill sponsors & lower better\\ \hline
$\mathbf{\mathcal{H}}$  & CI  & $45.19 \pm 0.9$  & $44.55 \pm 0.6$  & $25.38 \pm 1.2$  & $66.29 \pm 2.5$  & $51.17 \pm 0.2$ & 10.4\\ 
$\mathbf{X}$  & MLP  & $72.23 \pm 2.0$  & $58.75 \pm 0.7$  & $22.59 \pm 0.9$  & $65.52 \pm 1.5$  & $36.94 \pm 3.9$ & 10.4\\ 
$\mathbf{\mathcal{H}, X}$  & MLP+HLR  & $69.58 \pm 2.1$  & $65.13 \pm 1.8$  & $32.83 \pm 0.6$  & $73.49 \pm 0.8$  & $58.31 \pm 4.2$ & 8.8\\ 
$\mathbf{\mathcal{H}, X}$  & HGNN  & $74.25 \pm 2.1$  & $68.10 \pm 1.9$  & $41.57 \pm 1.0$  & $82.24 \pm 4.1$  & $81.86 \pm 5.5$ & 5.0\\ 
$\mathbf{\mathcal{H}, X}$  & HCHA  & $72.79 \pm 1.6$  & $65.26 \pm 0.3$  & $41.22 \pm 0.6$  & $81.16 \pm 3.2$ & $80.26 \pm 1.4$ & 6.8\\ 
$\mathbf{\mathcal{H}, X}$  & HyperGCN  & $75.19 \pm 2.0$  & $69.92 \pm 1.8$  & $30.80 \pm 0.1$  & $77.46 \pm 1.6$  & $75.83 \pm 4.9$ & 6.4\\ 
$\mathbf{\mathcal{H}, X}$  & HNHN  & $72.46 \pm 0.5$  & $67.03 \pm 1.1$  &  $41.88 \pm 0.9$ & $81.47 \pm 1.3$  & $52.82 \pm 0.7$ & 6.8\\ 
$\mathbf{\mathcal{H}, X}$  & AllSet  & $81.27 \pm 0.8$  & $67.25 \pm 0.4$ & $\underline{45.15 \pm 0.3}$ & $81.19 \pm 3.4$  & $78.99 \pm 3.1$ & 4.6\\ 
$\mathbf{\mathcal{H}, X}$  & ED-HNN & $\underline{85.17 \pm 0.1}$ & $69.91 \pm 0.7$ & $43.42 \pm 0.6$ & $84.77 \pm 1.1$ & $88.79 \pm 0.9$ & 3.0\\
$\mathbf{\mathcal{H}, X}$  & PhenomNN & $81.08 \pm 0.3$ & $\mathbf{72.79 \pm 1.3}$ & $42.74 \pm 0.3$ & $\underline{87.42 \pm 1.3}$& $\underline{88.94 \pm 1.2}$ & \underline{2.6}\\
\hline

$\mathbf{\mathcal{H}, X}$  & IHNN  & $\mathbf{87.61 \pm 0.4}$  &  $\underline{70.22 \pm 2.1}$ & $\mathbf{45.29 \pm 0.2}$  & $\mathbf{89.91 \pm 1.8}$ & $\mathbf{90.04 \pm 0.7}$& \textbf{1.2}\\ \hline
\end{tabular}
\end{table*}
To evaluate the performance of \textsc{IHNN} and all the baseline methods, we use the following hypergraph datasets on transductive node classification:
\begin{itemize}
    \item \textbf{DBLP} is a co-authorship dataset~\cite{yadati2019hypergcn}, where the nodes are authors and the hyperedges represent co-authorship relations. Node labels correspond to any of the 6 conference categories, namely `algorithms', `database', `programming', `datamining', `intelligence', and `vision'. For all the experiments, we use the train: test split given in~\cite{yadati2019hypergcn}. Note that each split has 70 to 30 train to test ratio.
    \item \textbf{Cora:} We use the Cora co-authorship dataset and the train: test splits created by~\cite{yadati2019hypergcn}. Here, the predictive task is again a multi-class node classification like DBLP.
    \item \textbf{Walmart-trips:} In this dataset~\cite{amburg2020clustering}, the hyperedges are sets of co-purchased products at Walmart, while the nodes denote the products themselves. The node labels denote the 11 broad categories the products, making the predictive task on this dataset also a multi-class node classification task.  
    \item \textbf{High-school:} This dataset, obtained from~\cite{benson2018simplicial} is constructed from interactions recorded by wearable sensors by students at a high school~\cite{amiri2018netgist}. The sensors record interactions at a resolution of 20 seconds. Each hyperedge corresponds to a group of students that were all in proximity of one another at a given time, based on data from sensors worn by students. Each node is labeled with the grade the student belongs to.
    \item \textbf{House-bills:} This dataset, obtained from~\cite{amburg2020clustering,chodrow2021generative,veldt2020minimizing} is constructed from the co-sponsorship of bills in the US House of Representatives. The nodes denote US Congress-persons, and the hyperedges denote the co-sponsors of bills put forth in the House of Representatives. Node label denotes the representative's political party affiliation.
\end{itemize}
The summary statistics of the datasets are given in Table~\ref{rel_view_data}.
\subsection{Baselines and Evaluation Metric}
\par\noindent\textbf{Baselines:} To evaluate the benefit of our proposed approach, we contrast IGNN against the following state-of-art non-trivial hypergraph representation learning baselines:
\begin{itemize}
    \item \textbf{Confidence Interval Method (CI):} Here the exact ranges of unlabeled vertices of a hypergraph are estimated through a confidence interval approach to determine the optimal solutions and a simpler subgradient method for solving the convex program~\cite{zhang2017re}. Note that this method only leverages the structure of the hypergraph to obtain the node embeddings.
    \item \textbf{Multilayer Perceptron (MLP):} This baseline ignores the hypergraph structure and only leverages the node features by considering each node to be an i.i.d. sample to make label predictions.
    \item \textbf{Multilayer Perceptron with Hypergraph Laplacian Regularization (MLP+HLR):} This baseline method~\cite{zhou2006learning} indirectly imposes the hypergraph structure by regularizing the embeddings learned by the MLP with a regularizer that depends on the hypergraph laplacian. 
    \item \textbf{Hypergraph Neural Network (HGNN):} This method~\cite{feng2019hypergraph} obtains the node embeddings by defining a heuristic hypergraph convolution module parameterized by the hypergraph Laplacian.
    \item \textbf{Hypergraph Convolution and Hypergraph Attention (HCHA):} Building on the message propagation through hypergraph convolution as in HGNN, this method~\cite{bai2021hypergraph} uses different degree normalizations and attention weights, with the attention weights depending on node features and the hyperedge features. However, the absence of explicit hyperedge features in the datasets prevents the application of attention to the hyperedges.
    \item \textbf{Hypergraph Networks with Hyperedge Neurons (HNHN):} This is again a variant~\cite{dong2020hnhn} of the hypergraph convolution combines a hypergraph convolution network with nonlinear activation functions and a normalization scheme to adjust the importance of high-cardinality hyperedges and high-degree vertices.
    \item \textbf{Graph Convolutional Networks on Hypergraphs (HyperGCN):} This method~\cite{yadati2019hypergcn} proposes a unique transformation method of a hypergraph to a graph and approximates each hyperedge of the hypergraph by a set of pairwise edges connecting the vertices of the hyperedge and treats the learning problem as a graph learning problem on the approximation.
    \item \textbf{AllSet Transformer (Allset):} This baseline~\cite{chienyou} uses an attention mechanism during the pooling of information during the hypergraph convolution step. However, this comes with an additional cost of a larger number of trainable parameters, which may not be optimized with less training data.
    \item \textbf{Equivariant Hypergraph Diffusion Neural Network(ED-HNN):} This baseline~\cite{wangequivariant} uses diffusion mechanisms to design a new hypergraph message passing structure.
    \item \textbf{Purposeful Hyperedges in Optimization Motivated Neural Networks(PhenomNN):} This baseline~\cite{wang2023hypergraph} defines parametrized hypergraph-regularized energy functions and demonstrates how their minimizers serve as learnable node embeddings.
\end{itemize}

\par\noindent\textbf{Evaluation Metric:} Like the prior works~\cite{chienyou,dong2020hnhn}, we use accuracy as the evaluation metric for all our experiments. While some related work~\cite{yadati2019hypergcn} report the mean test error, note that it is simply (100 - average accuracy).
\subsection{Experimental Setup}
We conducted all experiments on AMD EPYC 7763 64-Core Processor with 1.08 TB of memory and 8 NVIDIA A40 GPUs with CUDA version 12.1. Our code and experimental setup, including data construction, are available for peer review. \footnote{\url{https://github.com/Soothysay/IHNN}}. 
For all the datasets, the train: test split is 30:70, and each model's performance is averaged over 10 independent runs containing random data splits. For experiments on DBLP and Cora, all models are trained for 200 epochs, while High-school and House-Bills are trained for 1000 epochs.

\section{Results}
 We follow the transductive experimental setup described in Section 4. Note that the datasets chosen for our experiments come from different domains and significantly differ in structure. While Cora and DBLP are both co-authorship datasets, DBLP is much larger than Cora in terms of the number of nodes and the hyperedges. Additionally, each hyperedge in Cora is more dense than DBLP. On the other hand, High-school dataset is a contact-specific smaller dataset (like Cora) where the maximum number of hyperedges for a node is similar to that of DBLP. House-bills (bill sponsorship) dataset has the highest number of hyperedges connected to a given node and the associated task labels are binary. Finally, Walmart-trips is the largest dataset, both in terms of the total number of nodes and hyperedges. Additionally, the initial node features of Walmart-trips, High-school, and House-bills are randomly initialized, contrary to Cora and DBLP.  
\subsection{RQ1: Transductive Node Classification}
Our first experiment is a simple transductive node classification on the five datasets described earlier. The results are summarized in Table \ref{tab:main_result}. Overall, the results demonstrate that IHNN outperforms all the state-of-the-art baselines with an overall rank of 1.2, while the closest baseline has the overall rank of 2.6. 
This result demonstrates that IHNN is able to incorporate both node and edge information better than the baselines.

Since our experiment involves diverse datasets from different application domains, there are several interesting observations we make.
 First, let us consider the co-authorship networks, Cora and DBLP. AllSet Transformer, which is a generalized message-pooling mechanism, outperforms HGNN by 9.45\% on DBLP, while having a similar performance on Cora. However, IHNN outperforms AllSet by 7.80\% on DBLP and 4.41\% on Cora. Second, we notice that MLP outperforms CI on these two datasets but has worse performance for the 3 other datasets. This is because the co-authorship datasets have initial node features computed from the data while the initial node features were randomly initialized for the 3 other datasets. So, CI which uses the hypergraph structure to form node representations is not affected by input data noise like MLP in the 3 other datasets. Thirdly, we observe that IHNN significantly outperforms the nearest baseline on datasets where the number of hyperedges is much larger than the number of nodes (DBLP, High-school, House-bills). This demonstrates the ability of IHNN to capture long-range dependencies between nodes in hypergraphs. This empirical result is similar to the observation~\cite{gu2020implicit} previously made for standard graphs. Overall, IHNN outperforms the nearest baseline on 4 out of the 5 datasets. 

\subsection{RQ2: Ablation Study}
\begin{figure}
    \centering
    \includegraphics[width=0.8\linewidth]{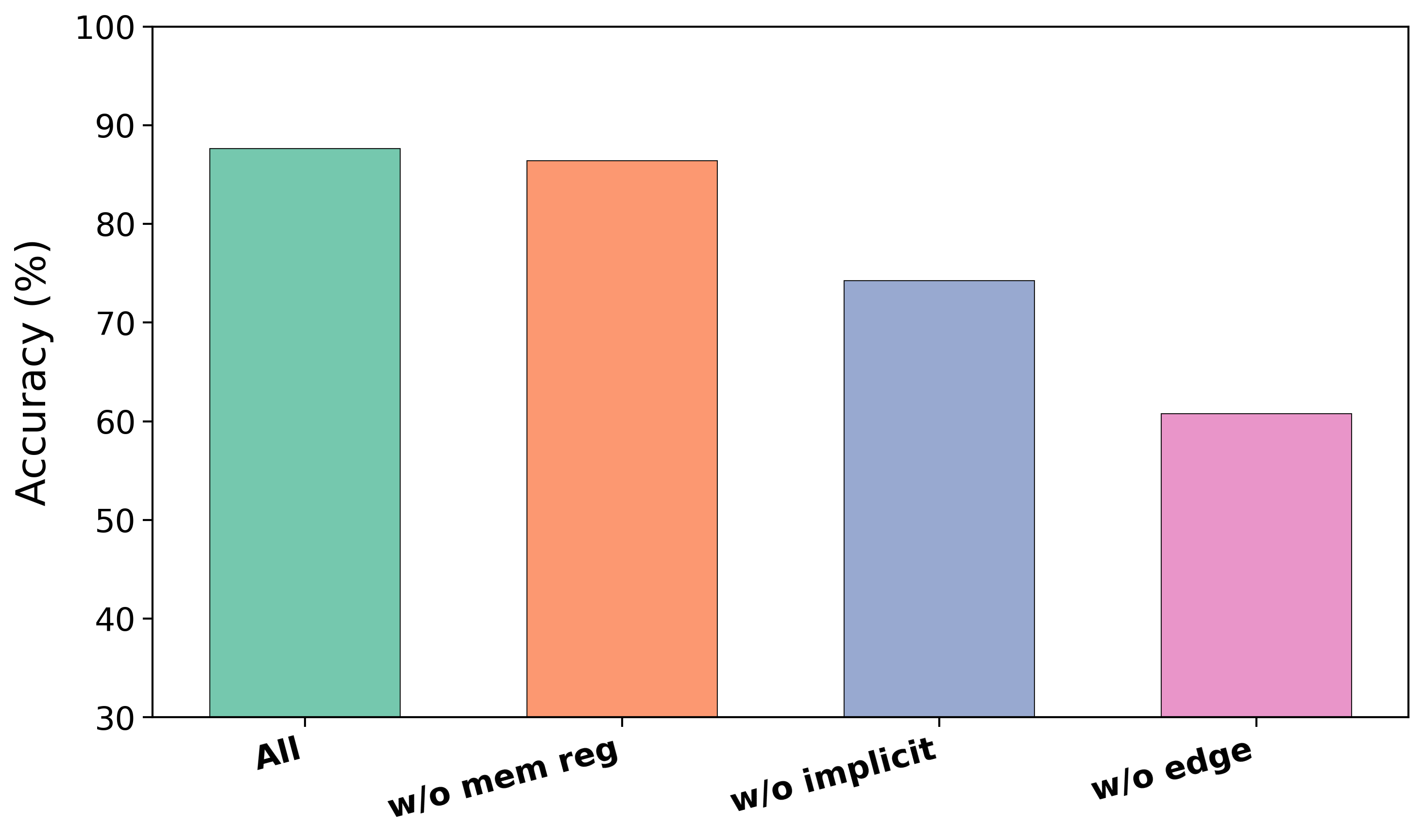}
    \caption{Ablation Study: Analysis of the individual contributions of the components of IHNN for DBLP.}
    \label{fig:ablation}
\end{figure}
The ablation study demonstrates the individual contribution of each component of IHNN towards the overall performance. To evaluate the contributions of the individual components, we ran experiments on the DBLP dataset dropping each component of IHNN at a time and evaluating the change in performance. Recall that the DBLP dataset contains a set of initial node features derived from the domain. The initial hyperedge features are constructed by mean-pooling the node features. 

The results are presented in Figure~\ref{fig:ablation}.  We display the mean test accuracy for the 30:70 train: test split averaged across 10 independent runs for each component drop. First of all, we notice that abandoning the implicit form (w/o implicit) results in about a 15.99\% drop in the accuracy of the overall framework, highlighting the need for an implicit model. Additionally, we highlight the importance of the cyclic form of our model and initial edge embeddings by randomly initializing edge features, which results in about a 32.90\% drop in performance. Finally, using IHNN without the membership regularizer (w/o mem reg) leads to a 1.39\% drop in performance. This shows that each individual component of our overall framework is important for the gain in performance.

Thus, we can conclude that each individual part of the overall framework makes a positive contribution to the overall performance. Additionally, the conjunction of all components leads to superior overall performance.
\subsection{RQ3: Sensitivity Analysis}
\begin{figure}[h!]
    \centering
    \includegraphics[width=\linewidth]{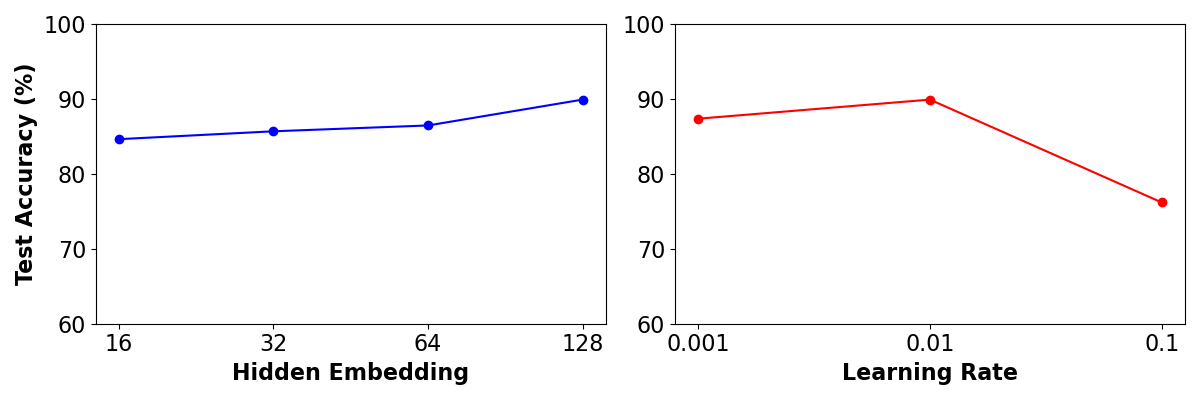}
    \caption{Sensitivity analysis of IHNN for the embedding size (left) and the learning rate (right). As observed, our proposed approach is relatively robust to the changes in embedding size  and the lower learning rate leads to better performance.}
    \label{fig:sensitivity}
\end{figure}
We explored the variation of the hyperparameters and tested performance by varying the latent embedding size by \{16, 32, 64, 128\} and the learning rate by \{0.1, 0.01, 0.001\} on the High-school dataset. The overall results of our experiments are presented in Figure~\ref{fig:sensitivity}. In the first experiment, we fix learning rate to 0.01 and vary the embedding size. We observe that the performance of IHNN is relatively stable to the variation of the hidden embedding size. Notably, varying the size embedding size from 32 to 64 leads to a 1.14\% gain in performance while changing from 64 to 128 brings a gain in performance by about 4\%. Second, by fixing the hidden embedding size to 128, and varying the learning rate, we observe that IHNN is stable to lower values of learning rate while the performance degrades considerably when the learning rate is set to 0.1. This means that a larger learning rate causes the overall optimization process to move away from the global minima causing degradation in performance. Thus we use a hidden embedding size of 128 and a learning rate of 0.01 for the experiments on the High-school dataset.
\subsection{RQ4: Case Study: Uncovering Medical Code Relationships from EHR Data}

To demonstrate the broader applicability of IHNN, we perform a case study on real-world health data. 
Hypergraphs have been used to encode the relationship between medical codes in Electronic Health Records (EHR) and also have been to improve the prediction of patient risks at an individual level~\cite{xu2023hypergraph,choudhuri2025domain}. Our case study focuses on leveraging IHNN to estimate patient risks and then analyzing the inferred relationships between medical codes.


We begin by extracting inpatient medical information and ICD-9 billing codes for patients admitted between January 1, 2011, and June 30, 2011 from a real-world EHR dataset (\textbf{R-EHR}) collected from the University of Iowa Hospitals and Clinics (UIHC)~\cite{jang2022risk,choudhuri2023continually}. Each patient visit is typically associated with multiple ICD-9 codes indicating multiple diagnoses during the hospital stay. ICD-9 codes are obtained from the patient's discharge summary. Our risk prediction task for this case study is readmission prediction. We constructed binary labels indicating whether the patient was readmitted to the hospital within 30 days of discharge. The label statistics are provided in Table~\ref{tab:readmission_counts}. 
\begin{table}[h]
    \centering
    \begin{tabular}{|c|c|}
        \hline
        \textbf{Readmission} & \textbf{Count} \\
        \hline
        False & 5448 \\
        True & 2987 \\
        \hline
    \end{tabular}
    \caption{Readmission Counts Statistics for the case study.}
    \label{tab:readmission_counts}
\end{table}

In the hypergraph we constructed, each unique patient visit to the healthcare facility is considered a node, and ICD-9 codes serve as hyperedges. This resulted in a hypergraph consisting of 8435 nodes and 3375 hyperedges. We trained \textsc{IHNN} with node classification loss on the readmission label. We randomly split visits to a 30:70 train: test ratio. Note that the label imbalance makes the ROC-AUC score a more suitable metric for evaluating this predictive performance~\cite{choudhuri2025domain}. \textsc{IHNN} obtained a test ROC-AUC of 68.39. 
Here, we are also interested in understanding if the relations between the hyper-edges (ICD-9 codes) inferred by \textsc{IHNN} correspond to domain knowledge. Firstly, we mapped the 3375 ICD-9 codes to 274 broad Clinical Classification Software (CCS) concepts. We then used the number of unique CCS concepts as the number of clusters to perform spectral clustering on the hyperedge embeddings and display the top 10 largest (in terms of number of elements) clusters with T-SNE visualization in Figure~\ref{fig:Case_study}. In the figure, the purple cluster is distant from the rest. On manual inspection, we found that the cluster contains the ICD-9 codes referring to adverse effects \& poisoning, cardiovascular diseases, respiratory, renal, neurological, gastrointestinal, urinary, and endocrine disorders. Detailed Descriptions of the meanings of these codes are given in Table~\ref{tab:icd9_codes}.
\begin{table}[h!]
    \centering
    \footnotesize
    \renewcommand{\arraystretch}{1.0}
    \begin{tabularx}{\columnwidth}{|c|X|}
        \hline
        \textbf{ICD-9 Code} & \textbf{Description} \\
        \hline
        E947.9  & Unspecified drug or medicinal substance causing adverse effects in therapeutic use \\
        440.21  & Atherosclerosis of native arteries of the extremities with intermittent claudication \\
        780.2   & Syncope and collapse \\
        276.4   & Mixed acid-base balance disorder \\
        599.71  & Gross hematuria \\
        786.05  & Shortness of breath \\
        788.41  & Urinary frequency \\
        599.70  & Hematuria, unspecified \\
        V12.71  & Personal history of peptic ulcer disease \\
        530.19  & Other esophagitis \\
        428.20  & Systolic heart failure, unspecified \\
        V45.76  & Presence of automatic (implantable) cardiac defibrillator \\
        429.9   & Heart disease, unspecified \\
        275.42  & Hypercalcemia \\
        820.8   & Fracture of unspecified part of neck of femur (hip) \\
        786.6   & Swelling, mass, or lump in chest \\
        996.66  & Infection and inflammatory reaction due to internal joint prosthesis \\
        E941.3  & Adverse effect of antineoplastic and immunosuppressive drugs \\
        446.5   & Giant cell arteritis \\
        438.89  & Other late effects of cerebrovascular disease \\
        411.89  & Other acute and subacute forms of ischemic heart disease \\
        398.90  & Rheumatic heart disease, unspecified \\
        710.0   & Systemic lupus erythematosus \\
        996.61  & Infection and inflammatory reaction due to cardiac device, implant, and graft \\
        331.82  & Dementia with Lewy bodies \\
        598.9   & Urethral stricture, unspecified \\
        429.83  & Takotsubo syndrome \\
        440.4   & Chronic total occlusion of artery of the extremities \\
        \hline
    \end{tabularx}
    \caption{ICD-9 Codes and Their Descriptions for the purple points.}
    \label{tab:icd9_codes}
\end{table}

Amongst these codes, we observed that eight of them are directly or indirectly related to heart attack. \textbf{411.89} is directly related, as ischemic heart disease~\cite{feske2021ischemic} often leads to myocardial infarction. \textbf{429.9} is also directly linked, as it encompasses various heart conditions, including those leading to heart attacks. Indirectly related codes include \textbf{428.20} since heart failure can result from repeated heart attacks or prolonged ischemia; \textbf{398.90}, as valve damage from rheumatic fever increases the risk of ischemic events~\cite{remenyi2016valvular}; \textbf{440.21} and \textbf{440.4} since atherosclerosis is a major cause of coronary artery disease and heart attacks~\cite{falk2006pathogenesis}; \textbf{429.83}, Takotsubo syndrome, which mimics a heart attack and is triggered by extreme stress~\cite{butt2022long}; and \textbf{996.61}, as infections from cardiac implants can lead to complications affecting heart function~\cite{klug2007risk}. These conditions either directly cause myocardial infarction or indicate underlying cardiovascular issues that increase heart attack risk.
\begin{table}[h!]
    \centering
    \footnotesize
    \renewcommand{\arraystretch}{1.0}
    \begin{tabularx}{\columnwidth}{|c|X|}
        \hline
        \textbf{ICD-9 Code} & \textbf{Description} \\
        \hline
        386.11  & Benign paroxysmal positional vertigo \\
        595.1   & Chronic interstitial cystitis \\
        784.69  & Other symbolic dysfunction (e.g., speech disturbance) \\
        296.44  & Bipolar I disorder, most recent episode manic, severe, with psychotic features \\
        376.01  & Orbital cellulitis \\
        706.1   & Acne vulgaris \\
        296.42  & Bipolar I disorder, most recent episode manic, moderate \\
        158.8   & Malignant neoplasm of other specified sites of peritoneum \\
        790.22  & Impaired fasting glucose \\
        357.0   & Acute infective polyneuritis (e.g., Guillain-Barré syndrome) \\
        305.43  & Combination of opioid abuse with other drug abuse \\
        270.2   & Other disturbances of amino-acid metabolism \\
        357.4   & Polyneuropathy in other diseases classified elsewhere \\
        295.40  & Schizophreniform disorder, unspecified \\
        986     & Toxic effect of carbon monoxide \\
        571.9   & Unspecified chronic liver disease \\
        530.6   & Gastroesophageal reflux disease (GERD) \\
        719.40  & Pain in unspecified joint \\
        716.59  & Unspecified polyarthritis involving multiple sites \\
        253.1   & Primary hyperaldosteronism \\
        305.30  & Cannabis abuse, unspecified \\
        437.8   & Other specified cerebrovascular diseases \\
        959.9   & Unspecified injury \\
        \hline
    \end{tabularx}
    \caption{ICD-9 Codes and Their Descriptions for the teal points.}
    \label{tab:icd9_codes_teal}
\end{table}

Similarly, the cluster of points colored in teal primarily contains ICD-9 codes associated with mental health and neurological issues. A detailed description of the ICD-9 codes and their meanings is given in Table~\ref{tab:icd9_codes_teal}. \textbf{386.11} is directly related, as benign paroxysmal positional vertigo (BPPV) affects the vestibular system, leading to dizziness and balance issues~\cite{bhattacharyya2008clinical}. \textbf{357.0} and \textbf{357.4} involve polyneuropathies~\cite{sommer2018polyneuropathies}, which affect the peripheral nervous system and can result in weakness, numbness, or pain. \textbf{784.69} is linked to speech disturbances, often indicative of neurological disorders. Mental health conditions are represented by \textbf{296.44} and \textbf{296.42}, which describe different severities of bipolar I disorder, impacting mood and cognition. \textbf{295.40} covers schizophreniform disorder, a condition with symptoms resembling schizophrenia. Substance-related disorders, such as \textbf{305.43} (opioid abuse with other drug abuse) and \textbf{305.30} (cannabis abuse), highlight the intersection between mental health and addiction. \textbf{437.8}, which includes cerebrovascular diseases, may contribute to cognitive impairment and neurological deficits. These conditions directly or indirectly affect brain function, mental well-being, and the nervous system. 

\begin{figure}
    \centering
    \includegraphics[width=0.9\linewidth]{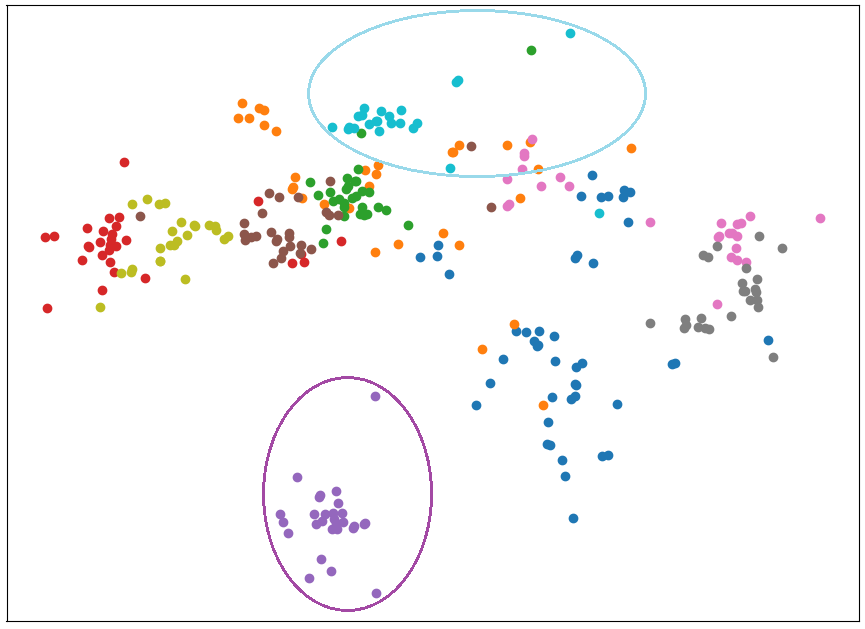}
    \caption{TSNE of edge embeddings for the top 10 clusters in R-EHR dataset. The colors represent the cluster types. In this case study we highlight the two clusters represented by the \textcolor{violet}{purple} and \textcolor{teal}{teal} circles.}
    \label{fig:Case_study}
\end{figure}
Thus, we observe that IHNN improves predictive performance by extracting meaningful relationships from real-world data. This provides additional interpretable insights behind the predictions made by IHNN in the medical domain and highlights the potential for the applicability of our proposed framework.

\section{Conclusion}
In this paper, we introduce the \textsc{Implicit Hypergraph Neural Network (IHNN)}, the first implicit model for learning hypergraph representations. \textsc{IHNN} jointly learns embeddings for both nodes and hyperedges and is guaranteed to converge. Additionally, it incorporates a novel membership loss, which has been empirically proven to enhance model performance. Our experimental results demonstrate that IHNN significantly outperforms state-of-the-art methods. Although the framework is sensitive to variations in learning rate, an additional case study shows the broader applicability of IHNN in solving real-world problems. Future directions of research include further exploration of additional regularization methods that work well in conjunction with implicit models, as well as connecting contrastive augmentations with IGNN.
\section{Acknowledgements}
This project is partially supported by the NSF Career Award IIS 2442159. Its contents are solely the responsibility of the authors and do not necessarily represent the official views of the NSF.
\bibliographystyle{IEEEtran}
\bibliography{bibliography}


\end{document}